\documentclass[runningheads]{llncs}

\usepackage{natbib}
\usepackage{hyperref}
\usepackage{xcolor}
\usepackage{amsmath}
\usepackage{amssymb}

\usepackage{graphicx}
\usepackage{booktabs}
\usepackage{multirow}
\usepackage{verbatim}
\usepackage{wrapfig}

\usepackage{tabularx}
\usepackage{tikz}
\usetikzlibrary{matrix}

\DeclareMathOperator*{\logit}{logit}

\begin{document}
\title{Squashed Shifted PMI Matrix: Bridging Word Embeddings and Hyperbolic Spaces}
%
%
\author{Zhenisbek Assylbekov\orcidID{0000-0003-0095-9409} \and
Alibi Jangeldin}
\authorrunning{F. Author and S. Author}
%
\institute{School of Sciences and Humanities, Nazarbayev University, Nur-Sultan, Kazakhstan\\\email{zhassylbekov@nu.edu.kz}}
\maketitle              
\begin{abstract}
We show that removing sigmoid transformation in the skip-gram with negative sampling (SGNS) objective does not harm the quality of word vectors significantly and at the same time is related to factorizing a squashed shifted PMI matrix which, in turn, can be treated as a connection probabilities matrix of a random graph. Empirically, such graph is a complex network, i.e. it has strong clustering and scale-free degree distribution, and is tightly connected with hyperbolic spaces. In short, we show the connection between static word embeddings and hyperbolic spaces through the squashed shifted PMI matrix using analytical and empirical methods.

\end{abstract}
\section{Introduction}
Modern word embedding models \citep{mccann2017learned,peters2018deep,devlin2019bert} build vector representations of words in context, i.e. the same word will have different vectors when used in different contexts (sentences). Earlier models \citep{mikolov2013distributed,pennington2014glove} built the so-called static embeddings: each word was represented by a single vector, regardless of the context in which it was used. 

Despite the fact that static word embeddings are considered obsolete today, they have several advantages compared to contextualized ones. Firstly, static embeddings are trained much faster (few hours instead of few days) and do not require large computing resources (1 consumer-level GPU instead of 8--16 non-consumer GPUs). Secondly, they have been studied theoretically in a number of works \citep{levy2014neural,arora2016latent,hashimoto2016word,gittens2017skip,tian2017mechanism,ethayarajh2019towards,allen2019vec,allen2019analogies,assylbekov2019context,zobnin2019learning} but not much has been done for the contextualized embeddings \citep{reif2019visualizing}. Thirdly, static embeddings are still an integral part of deep neural network models that produce contextualized word vectors, because embedding lookup matrices are used at the input and output (softmax) layers of such models. Therefore, we consider it necessary to further study static  embeddings. 

With all the abundance of both theoretical and empirical studies on static vectors, they are not fully understood, as this work shows. For instance, it is generally accepted that good quality word vectors are inextricably linked with a low-rank approximation of the pointwise mutual information (PMI) matrix or the Shifted PMI (SPMI) matrix, but we show that vectors of comparable quality can also be obtained from a low-rank approximation of a \textit{Squashed} SPMI matrix (Section~\ref{sec:bpmi_wv}). Thus, a Squashed SPMI matrix is a viable alternative to standard PMI/SPMI matrices when it comes to obtaining word vectors. 

At the same time, it is easy to interpret the Squashed SPMI matrix with entries in $[0, 1)$ as a connection probabilities matrix for generating a random graph. Studying the properties of such a graph, we come to the conclusion that it is a so-called complex network, i.e. it has a strong clustering property and a scale-free degree distribution (Section~\ref{sec:bpmi_cn}). 

It is noteworthy that complex networks, in turn, are dual to hyperbolic spaces (Section~\ref{sec:cn_hg}) as was shown by \citet{krioukov2010hyperbolic}. Hyperbolic geometry has been used to train word vectors \citep{nickel2017poincare,tifrea2018poincar} and has proven its suitability --- in a hyperbolic space, word vectors need lower dimensionality than in the Euclidean space.

Thus, to the best of our knowledge, this is the first work that establishes simultaneously a connection between word vectors, a Squashed SPMI matrix, complex networks, and hyperbolic spaces. Figure~\ref{fig:summary} summarizes our work and serves as a guide for the reader.

\begin{figure}
    \centering
\begin{tikzpicture}
  \matrix (m) [matrix of math nodes,row sep=5em,column sep=5em,minimum width=10em]
  {
     \text{Squashed SPMI} & \text{Complex Networks} \\
     \text{Word Embeddings} & \text{Hyperbolic Spaces} \\};
  \path
    (m-2-1) edge node [left] {Section~\ref{sec:bpmi_wv}} (m-1-1)
    (m-1-1) edge node [below] {Section~\ref{sec:bpmi_cn}} (m-1-2)
    (m-1-2) edge node [right] {Section~\ref{sec:cn_hg}} (m-2-2);
\end{tikzpicture}
    \caption{Summary of our work}
    \label{fig:summary}
\end{figure}
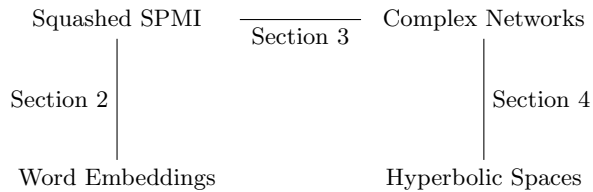

\subsection*{Notation}
We let $\mathbb{R}$ denote the real numbers. Bold-faced lowercase letters ($\mathbf{x}$) denote vectors, plain-faced lowercase letters ($x$) denote scalars, $\langle\mathbf{x},\mathbf{y}\rangle$ is the Euclidean inner product, $(a_{ij})$ is a matrix with the $ij$-th entry being $a_{ij}$. `i.i.d.' stands for `independent and identically distributed'. We use the sign $\propto$ to abbreviate `proportional to', and the sign $\sim$ to abbreviate `distributed as'.

Assuming that words have already been converted into indices, let $\mathcal{W}:=\{1,\ldots,n\}$ be a finite vocabulary of words. Following the setup of the widely used {\sc word2vec} model \citep{mikolov2013distributed}, we use \textit{two} vectors per each word $i$: (1) $\mathbf{w}_i\in\mathbb{R}^d$ when $i\in\mathcal{W}$ is a center word, (2) $\mathbf{c}_i\in\mathbb{R}^d$ when $i\in\mathcal{W}$ is a context word; and we assume that $d\ll n$. 

In what follows we assume that our dataset consists of co-occurence pairs $(i,j)$. We say that ``the words $i$ and $j$ co-occur'' when they co-occur in a fixed-size window of words. 
The number of such pairs, i.e. the size of our dataset, is denoted by $N$. Let $\#(i,j)$ be the number of times the words $i$ and $j$ co-occur, then $N=\sum_{i\in\mathcal{W}}\sum_{j\in\mathcal{W}}\#(i,j)$.

\section{Squashed SPMI and Word Vectors}\label{sec:bpmi_wv}
A well known skip-gram with negative sampling (SGNS) word embedding model of \citet{mikolov2013distributed} maximizes the following objective function
\begin{equation}\textstyle
    \sum_{i\in\mathcal{W}}\sum_{j\in\mathcal{W}}\#(i,j)\left(\log\sigma(\langle\mathbf{w}_i,\mathbf{c}_j\rangle)+k\cdot\mathbb{E}_{j'\sim p}[\log
    \sigma(-\langle\mathbf{w}_i,\mathbf{c}_{j'}\rangle)]\right),\label{eq:sgns}
\end{equation}
where $\sigma(x)=\frac{1}{1+e^{-x}}$ is the logistic sigmoid function, $p$ is a smoothed unigram probability distribution for words\footnote{The authors of SGNS suggest $p(i)\propto\#(i)^{3/4}$.}, and $k$ is the number of negative samples to be drawn. Interestingly, training SGNS is approximately equivalent to finding a low-rank approximation of a Shifted PMI matrix \citep{levy2014neural} in the form $\log\frac{p(i,j)}{p(i)p(j)}-\log k\approx\langle\mathbf{w}_i,\mathbf{c}_j\rangle$,
where the left-hand side is the 
$ij$-th element of the $n\times n$ shifted PMI matrix, and the right-hand side is an element of a matrix with rank $\le d$ since $\mathbf{w}_i,\mathbf{c}_j\in\mathbb{R}^d$. This approximation (up to a constant shift) was later re-derived by \citet{arora2016latent,assylbekov2019context,allen2019vec,zobnin2019learning} under different sets of assuptions. In this section we show that constraint optimization of a slightly modified SGNS objective \eqref{eq:sgns} leads to a low-rank approximation of the  \textit{Squashed} Shifted PMI ($\sigma$SPMI) matrix, defined as $\sigma\mathrm{SPMI}_{ij}:=\sigma(\mathrm{PMI}_{ij}-\log k)$.

\begin{theorem}\label{thm:main}
Assuming $0<\langle\mathbf{w}_i,\mathbf{c}_j\rangle<1$, the following objective function
\begin{equation}
    \mathcal{L}=\sum_{i\in\mathcal{W}}\sum_{j\in\mathcal{W}}\underbrace{\#(i,j)\left(\log\langle\mathbf{w}_i,\mathbf{c}_j\rangle+ k\cdot\mathbb{E}_{j'\sim P}[\log(1-\langle\mathbf{w}_i,\mathbf{c}_{j'}\rangle)]\right)}_{\ell(\mathbf{w}_i,\mathbf{c}_j)},\label{eq:rsgns}
\end{equation}
reaches its optimum at $\langle\mathbf{w}_i,\mathbf{c}_j\rangle=\sigma\mathrm{SPMI}_{ij}$.
\end{theorem}
\begin{proof}
Expanding the sum and the expected value in \eqref{eq:rsgns} as in \citet{levy2014neural}, and defining $p(i,j):=\frac{\#(i,j)}{N}$, $p(i):=\frac{\#(i)}{N}$, we have
\begin{equation}
\mathcal{L}=N\sum_{i\in\mathcal{W}}\sum_{j\in\mathcal{W}}p(i,j)\cdot\log\langle\mathbf{w}_i,\mathbf{c}_j\rangle+p(i)\cdot p(j)\cdot k\cdot\log(1-\langle\mathbf{w}_i,\mathbf{c}_{j}\rangle).
\end{equation}
Thus, we can rewrite the individual objective $\ell(\mathbf{w}_i,
\mathbf{c}_j)$ in \eqref{eq:rsgns} as
\begin{equation}
    \ell=N\left[p(i,j)\cdot\log\langle\mathbf{w}_i,\mathbf{c}_j\rangle+ p(i)\cdot p(j)\cdot k\cdot\log(1-\langle\mathbf{w}_i,\mathbf{c}_{j}\rangle)\right].\label{eq:rsgns2}
\end{equation}
Differentiating \eqref{eq:rsgns2} w.r.t. $\langle\mathbf{w}_i,
\mathbf{c}_j\rangle$ we get
$$
\frac{\partial\ell}{\partial\langle\mathbf{w}_i,\mathbf{c}_j\rangle}=N\left[\frac{p(i,j)}{\langle\mathbf{w}_i,\mathbf{c}_j\rangle}-\frac{p(i)\cdot p(j)\cdot k}{1-\langle\mathbf{w}_i,\mathbf{c}_j\rangle}\right].
$$
Setting this derivative to zero gives
\begin{multline}
    \frac{p(i,j)}{p(i)p(j)}\cdot\frac1k=\frac{\langle\mathbf{w}_i,\mathbf{c}_j\rangle}{1-\langle\mathbf{w}_i,\mathbf{c}_j\rangle}\quad\Rightarrow\quad\log\frac{p(i,j)}{p(i)p(j)}-\log k=\log\frac{\langle\mathbf{w}_i,\mathbf{c}_j\rangle}{1-\langle\mathbf{w}_i,\mathbf{c}_j\rangle}\\
    \Leftrightarrow\quad\log\frac{p(i,j)}{p(i)p(j)}-\log k=\logit\langle\mathbf{w}_i,\mathbf{c}_j\rangle\\\Leftrightarrow\quad\sigma\left(\log\frac{p(i,j)}{p(i)p(j)}-\log k\right)=\langle\mathbf{w}_i,\mathbf{c}_j\rangle,\label{eq:factorization}
\end{multline}
where $\logit(q):=\log\frac{q}{1-q}$ is the logit function which is the inverse of the logistic sigmoid function, i.e. $\sigma(\logit(q))=q$. From \eqref{eq:factorization} we have $\sigma\mathrm{SPMI}_{ij}=\langle\mathbf{w}_i,\mathbf{c}_j\rangle$,
which concludes the proof.
\end{proof}
\begin{remark}
Since $\sigma(x)$ can be regarded as a smooth approximation of the Heaviside step function $H(x)$, defined as $H(x)=1$ if $x>0$ and $H(x)=0$ otherwise, it is tempting to consider a \textit{binarized} SPMI (BSPMI) matrix $H(\mathrm{PMI}_{ij}-\log k)$ instead of $\sigma$SPMI. Being a binary matrix, BSPMI can be interpreted as an adjacency matrix of a graph, however our empirical evaluation below (Table~\ref{tab:emb_eval}) shows that such strong roughening of the $\sigma$SPMI matrix degrades the quality of the resulting word vectors. This may be due to concentration of the SPMI values near zero (Figure~\ref{fig:spmi_hyperbolic}), while $\sigma(x)$ is approximated by $H(x)$ only for $x$ away enough from zero.
\end{remark}
\begin{remark}
The objective \eqref{eq:rsgns} differs from the SGNS objective \eqref{eq:sgns} only in that the former does not use the sigmoid function (keep in mind that $\sigma(-x)=1-\sigma(x)$). We will refer to the objective \eqref{eq:rsgns} as \textit{Nonsigmoid SGNS}.
\end{remark}

\subsection*{Direct Matrix Factorization}
Optimization of the Nonsigmoid SGNS \eqref{eq:rsgns} is not the only way to obtain a low-rank approximation of the $\sigma$SPMI matrix. A viable alternative is factorizing the $\sigma$SPMI matrix with the singular value decomposition (SVD):
$\sigma\mathrm{SPMI}=\mathbf{U}\mathbf{\Sigma}\mathbf{V}^\top$, with orthogonal $\mathbf{U},\mathbf{V}\in\mathbb{R}^{n\times n}$ and diagonal $\mathbf{\Sigma}\in\mathbb{R}^{n\times n}$, and then zeroing out the $n-d$ smallest singular values, i.e.
\begin{equation}
\sigma\mathrm{SPMI}\approx\mathbf{U}_{1:n,1:d}\mathbf{\Sigma}_{1:d,1:d}\mathbf{V}^\top_{1:d,1:n},\label{eq:truncated_svd}    
\end{equation}
where we use $\mathbf{A}_{a:b,c:d}$ to denote a submatrix located at the intersection of rows $a, a+1, \ldots, b$ and columns
$c, c + 1, \ldots, d$ of $\mathbf{A}$. By the Eckart-Young theorem \citep{eckart1936approximation}, the right-hand side of \eqref{eq:truncated_svd} is the closest rank-$d$ matrix to the $\sigma$SPMI matrix in Frobenius norm. The word and context embedding matrices can be obtained from \eqref{eq:truncated_svd} by setting 
$\mathbf{W}^{\text{SVD}}:=\mathbf{U}_{1:n,1:d}\sqrt{\mathbf{\Sigma}_{1:d,1:d}}$, and $\mathbf{C}^{\text{SVD}}:=\sqrt{\mathbf{\Sigma}_{1:d,1:d}}\mathbf{V}^\top_{1:d,1:n}$. When this is done for a positive SPMI (PSPMI) matrix, defined as $\max(\mathrm{PMI}_{ij}-\log k, 0)$, the resulting word embeddings are comparable in quality with those from the SGNS \citep{levy2014neural}. 

\subsection*{Empirical Evaluation of the $\sigma$SPMI-based Word Vectors}
To evaluate the quality of word vectors resulting from the Nonsigmoid SGNS objective and $\sigma$SPMI factorization, we use the well-known corpus, \texttt{text8}.\footnote{\url{http://mattmahoney.net/dc/textdata.html}.} 
We ignored words that appeared less than 5 times, resulting in a vocabulary of 71,290 words. The SGNS and Nonsigmoid SGNS embeddings were trained using our custom implementation.\footnote{\url{https://github.com/zh3nis/SGNS}
} The SPMI matrices were extracted using the {\sc hyperwords} tool of \citet{levy2015improving} and the truncated SVD was performed using the {\sc scikit-learn} library of \citet{scikit-learn}.
\setlength{\tabcolsep}{4pt}
\begin{table}[htbp]
\caption{Evaluation of word embeddings on the analogy tasks (Google and MSR) and on the similarity tasks (the rest). For word similarities evaluation metric is the Spearman's correlation with the human ratings, while for word analogies it is the percentage of correct answers.}
\begin{center}
\begin{tabularx}{\textwidth}{l | c c c c | c c}
\toprule
Method & WordSim & MEN  & M. Turk & Rare Words & Google & MSR \\
\midrule
 SGNS & \textbf{.678}  & \textbf{.656} & .690 & \textbf{.334}  & \textbf{.359} & \textbf{.394} \\
 Nonsigm. SGNS & .649 & .649 & \textbf{.695} & .299 & .330 & .330 \\
 \midrule
 PMI + SVD & \textbf{.663} & \underline{.667} & \textbf{.668} & \textbf{.332} & \textbf{.315} & \underline{.323} \\
 SPMI + SVD & .509 & .576 & .567 & .244 & .159 & .107 \\
 PSPMI + SVD & .638 & \textbf{.672} & .658 & .298 & .246 & .207\\
 $\sigma$SPMI + SVD & \underline{.657} & .631 & \underline{.661} & \underline{.328} & \underline{.294} & \textbf{.341} \\
 BSPMI + SVD & .623 & .586 & .643 & .278 & .177 & .202\\
\bottomrule
\end{tabularx}
\end{center}
\label{tab:emb_eval}
\end{table}
The trained embeddings were evaluated on several word similarity and word analogy tasks: {\sc WordSim} \citep{finkelstein2002placing}, {\sc MEN} \citep{bruni2012distributional}, {\sc M.Turk} \citep{radinsky2011word}, {\sc Rare Words} \citep{luong2013better}, {\sc Google} \citep{mikolov2013efficient}, and {\sc MSR} \citep{mikolov2013linguistic}. We used the {\sc Gensim} tool of \citet{rehurek_lrec} for evaluation. For answering analogy questions ($a$ is to $b$ as $c$ is to $?$) we use the {\sc 3CosAdd} method of \citet{levy2014linguistic} and the evaluation metric for the analogy questions is the percentage of correct answers. We mention here that our goal is not to beat state of the art, but to compare SPMI-based embeddings (SGNS and SPMI+SVD) versus $\sigma$SPMI-based ones (Nonsigmoid SGNS and $\sigma$SPMI+SVD). The results of evaluation are provided in Table \ref{tab:emb_eval}. 

As we can see the Nonsigmoid SGNS embeddings in general underperform the SGNS ones but not by a large margin. $\sigma$SPMI shows a competitive performance among matrix-based methods across most of the tasks. Also, Nonsigmoid SGNS and $\sigma$SPMI demonstrate comparable performance as predicted by Theorem~\ref{thm:main}. Although BSPMI is inferior to $\sigma$SPMI, notice that such aggressive compression as binarization still retains important information on word vectors.

\begin{figure}[htbp]
\begin{minipage}{.67\textwidth}
\begin{center}
\includegraphics[width=.49\textwidth]{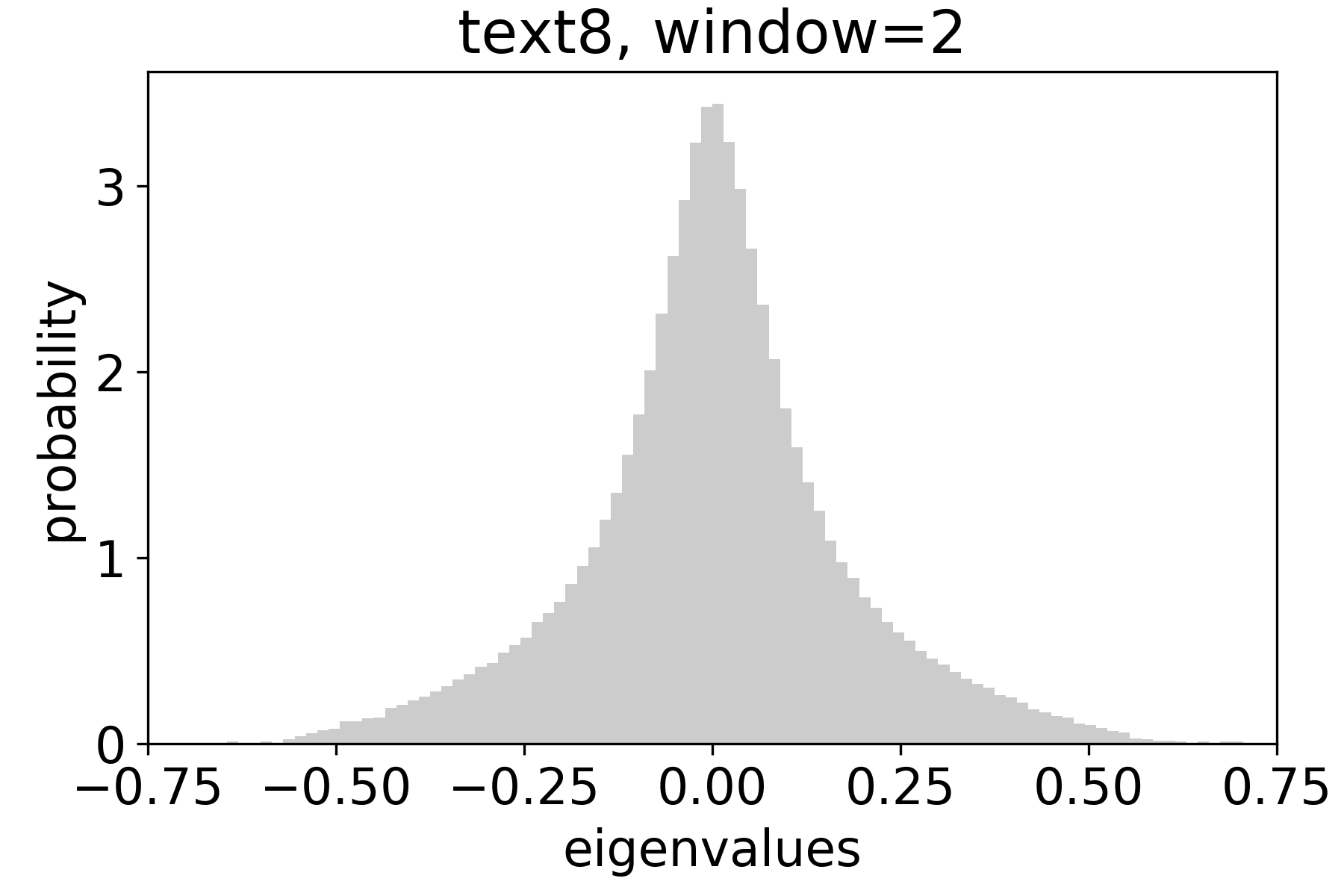}
\includegraphics[width=.49\textwidth]{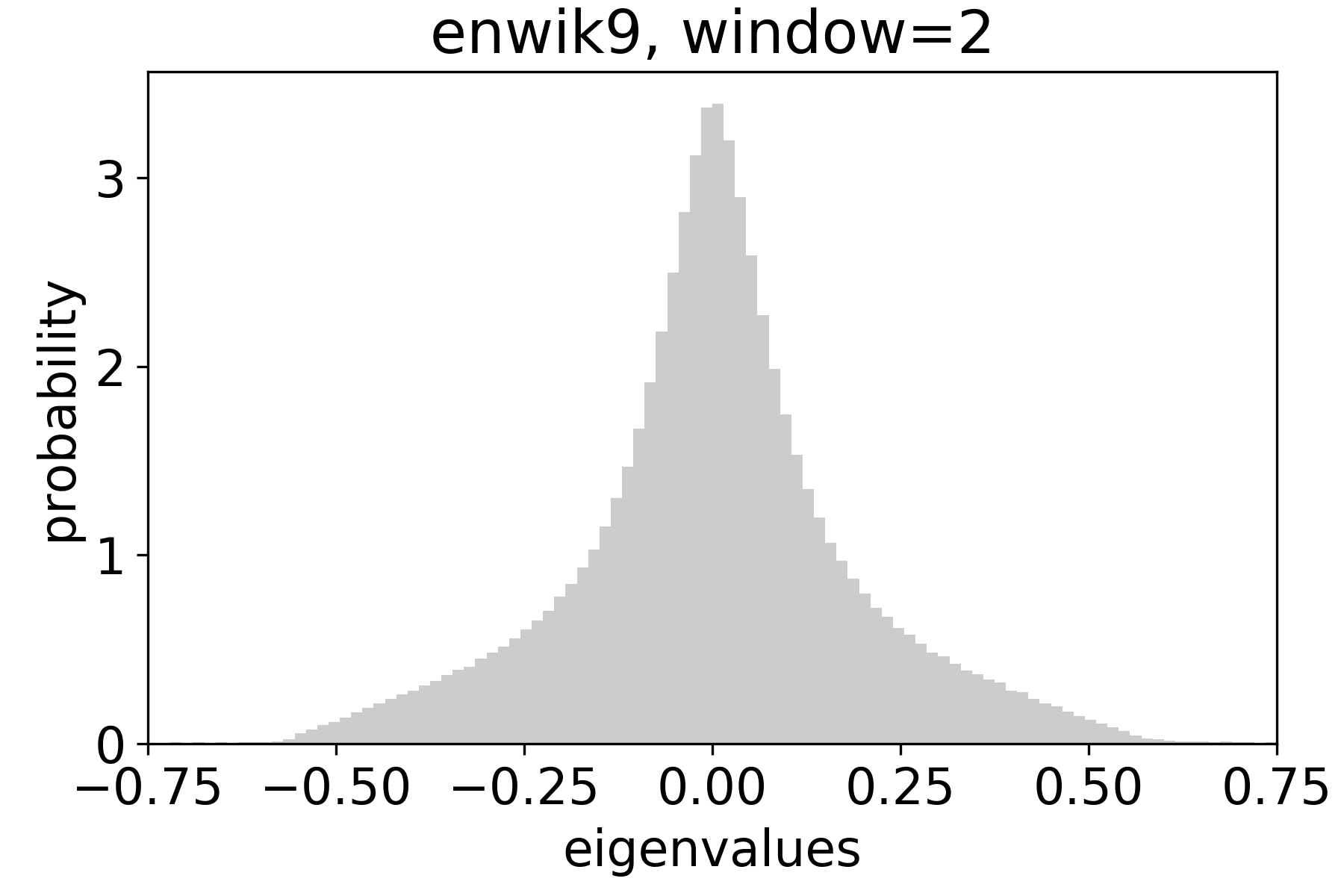}\\
\includegraphics[width=.49\textwidth]{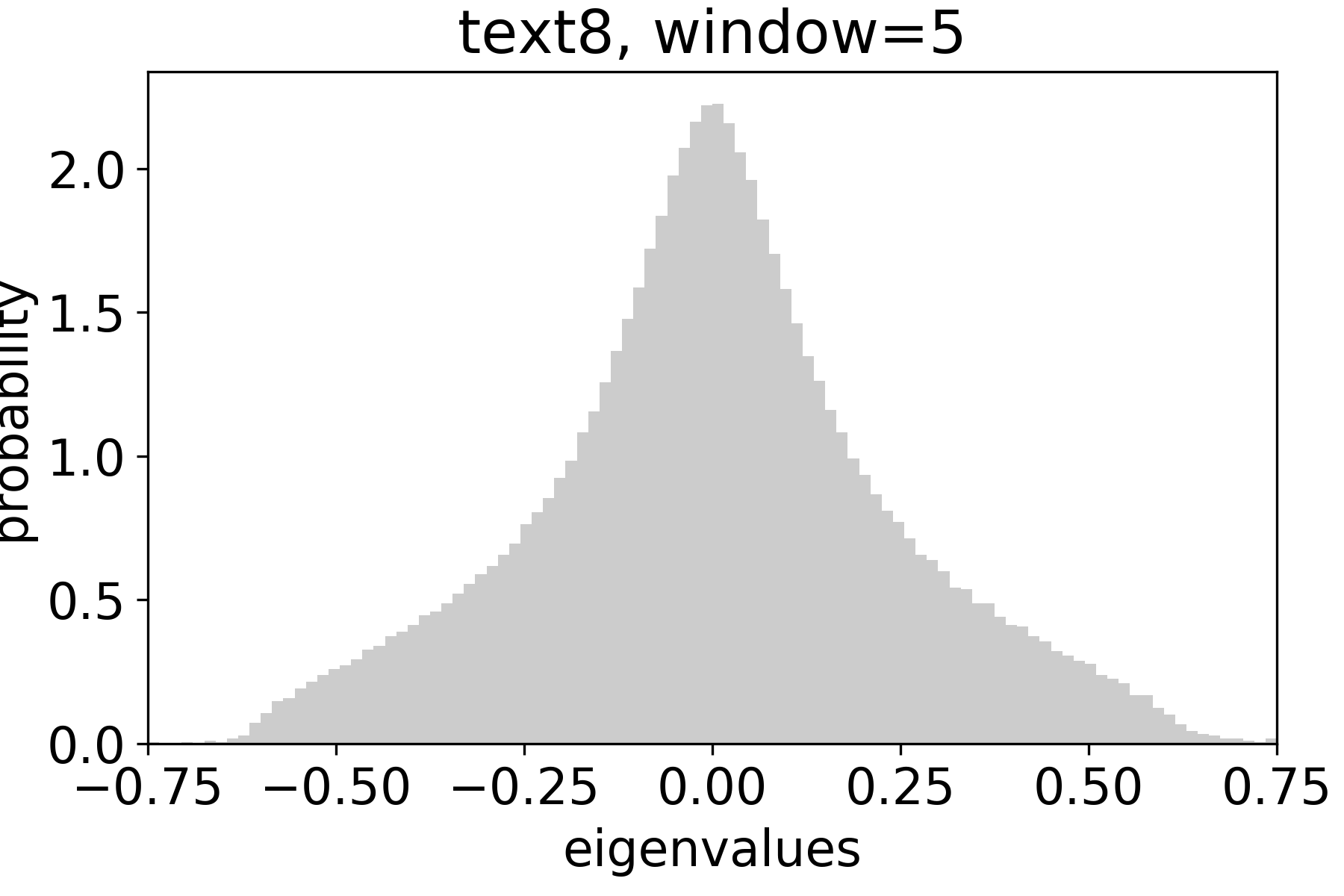}
\includegraphics[width=.49\textwidth]{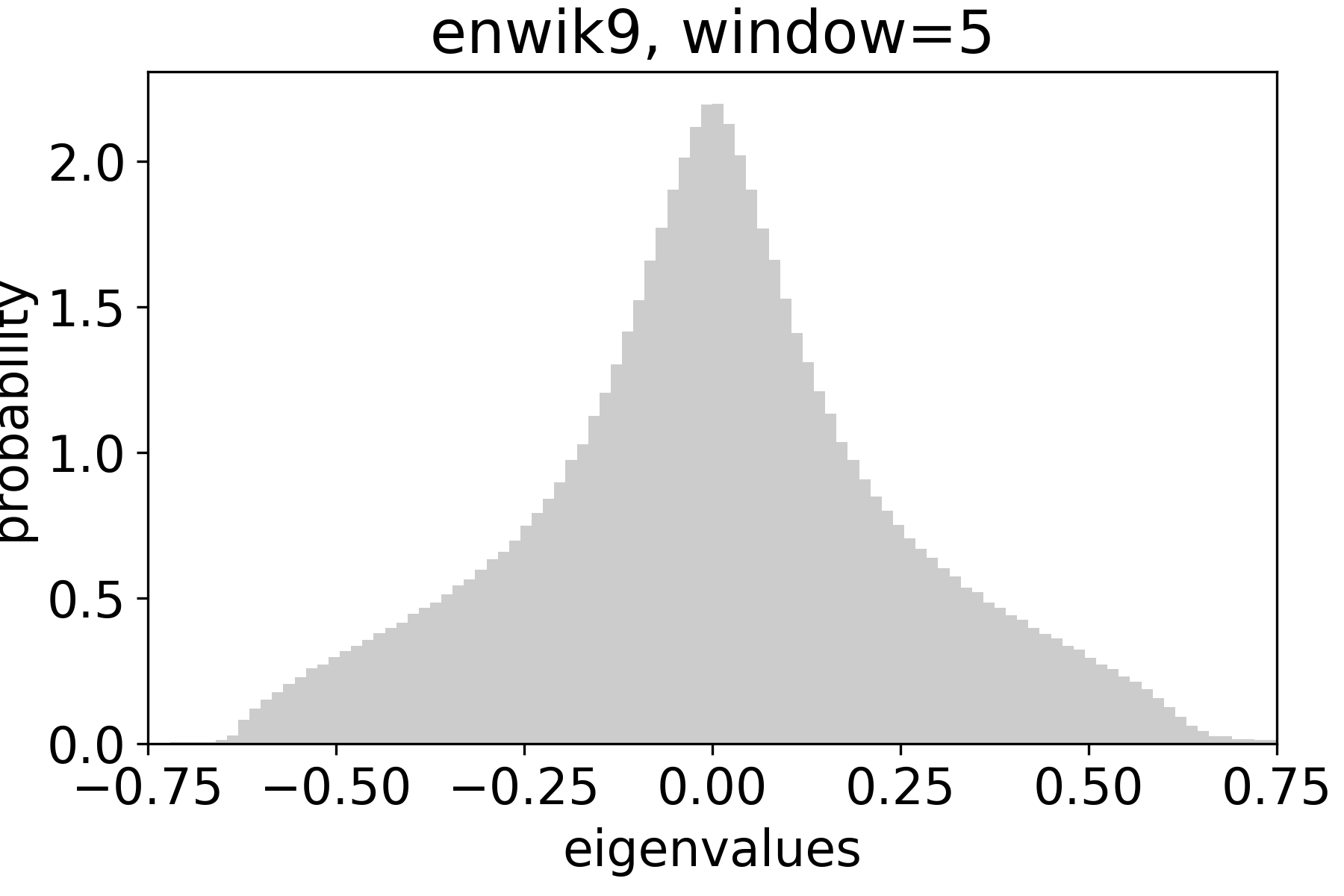}
\end{center}
\end{minipage}\begin{minipage}{.33\textwidth}
\includegraphics[width=\textwidth]{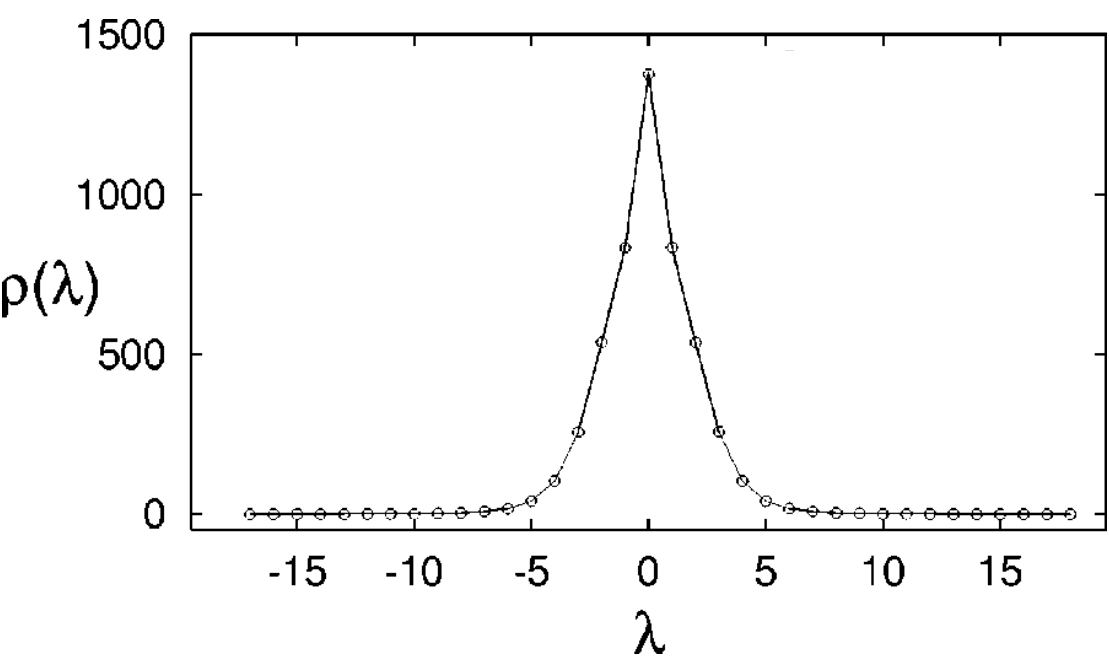}\\\includegraphics[width=\textwidth]{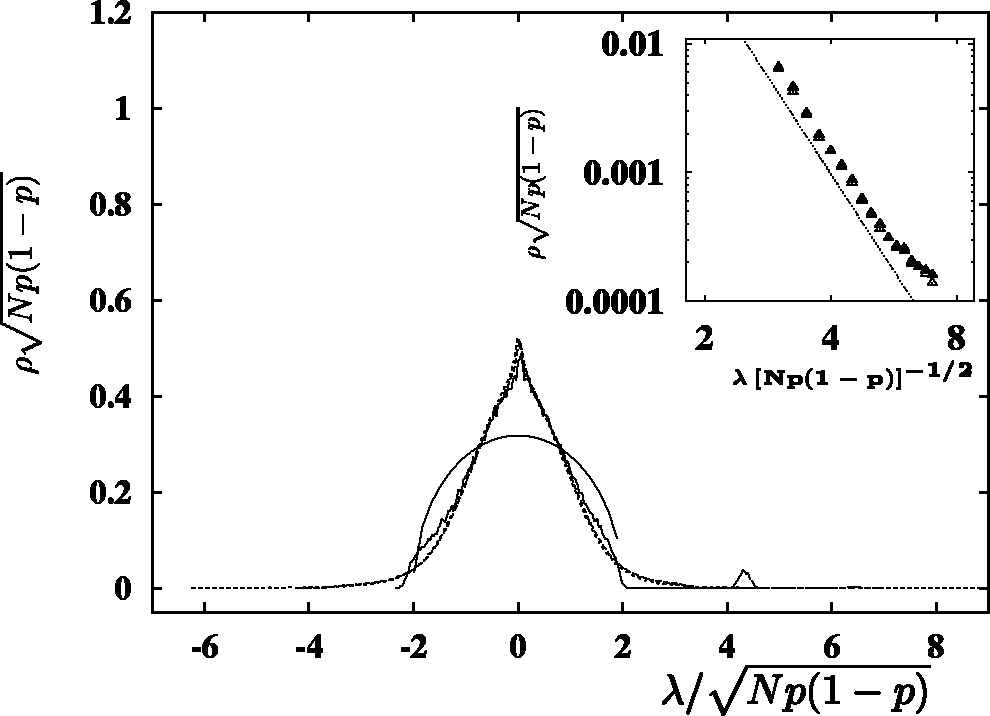}
\end{minipage}
\caption{Spectral distribution of the $\sigma$SPMI-induced graphs (left and middle columns), and of scale-free random graphs with strong clustering property (right top: \citet{goh2001spectra}, right bottom: \citet{farkas2001spectra}). When generating several random graphs from the same $\sigma$SPMI matrix, their eigenvalue distributions are visually indistinguishable, thus we display the results of one run per each matrix.}
\label{fig:eigvals}
\end{figure}
\begin{figure}[htbp]
\centering
\includegraphics[width=.45\textwidth]{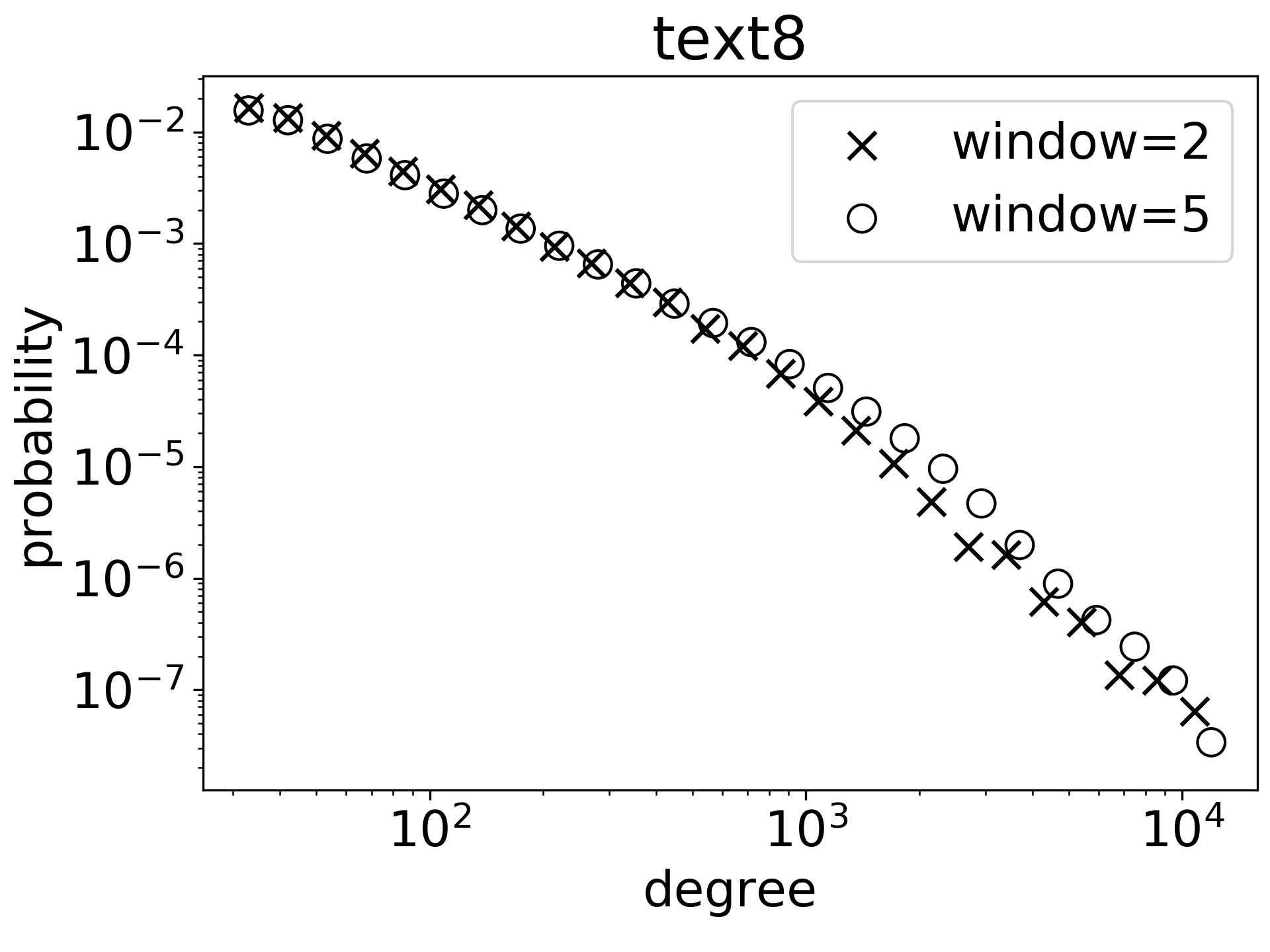}\hfill\includegraphics[width=.45\textwidth]{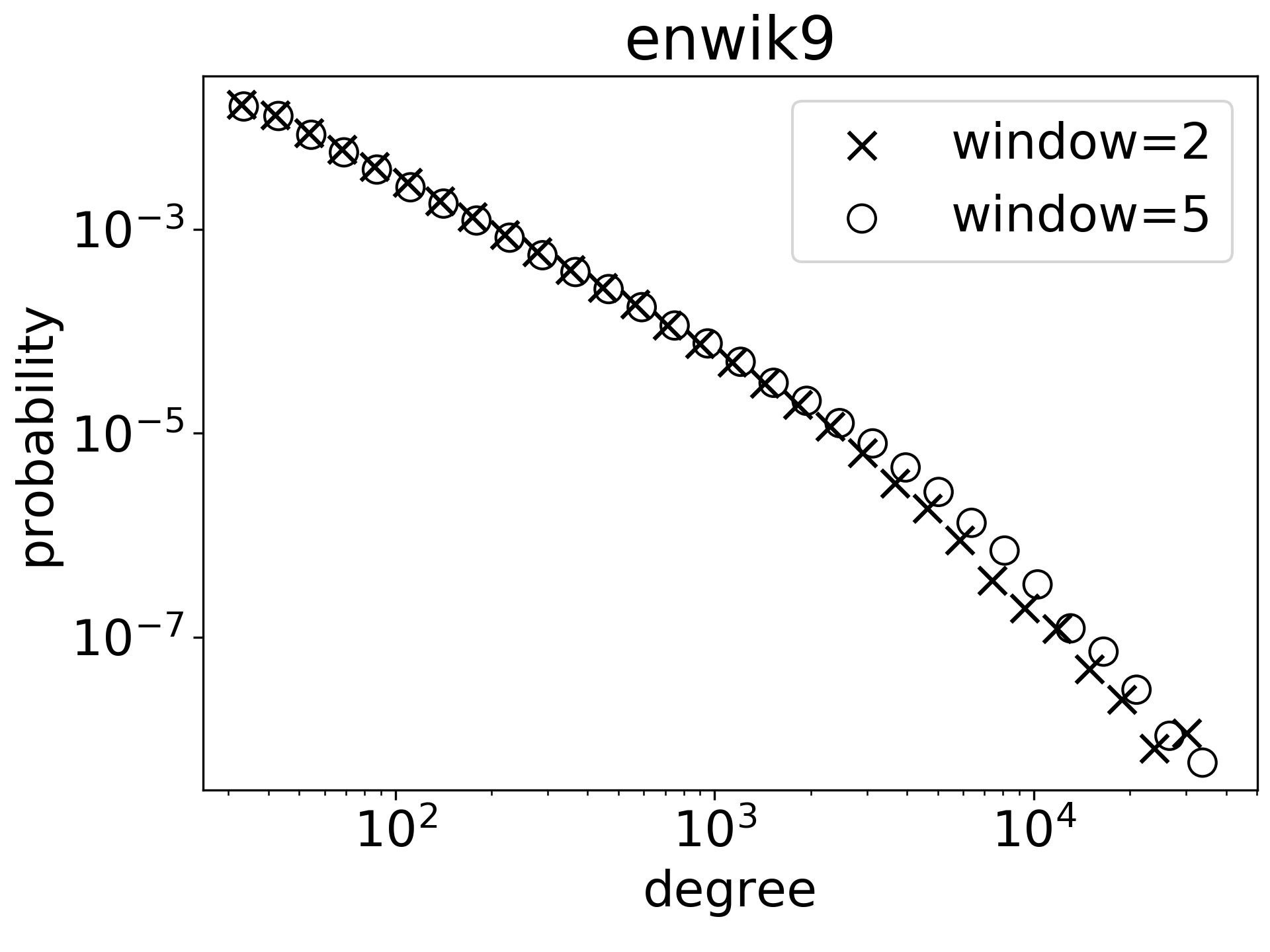}
\caption{Degree distributions of the $\sigma$SPMI-induced graphs. The axes are on logarithmic scales.}
\label{fig:degree_dist}
\end{figure}
\setlength{\tabcolsep}{9pt}
\begin{table}[htbp]
\caption{Clustering coefficients of the $\sigma$SPMI-induced graphs. For each corpus--window combination we generate ten graphs and report 95\% confidence intervals across these ten runs.}
    \centering
    \begin{tabularx}{\textwidth}{c c c c c}
        \toprule
         & \multicolumn{2}{c}{\texttt{text8}} & \multicolumn{2}{c}{\texttt{enwik9}} \\
         & window $=2$ & window $=5$ & window $=2$ & window $=5$\\
         \midrule
        $C$ & $.1341\pm.0006$ & $.1477\pm.0005$ & $.1638\pm.0006$ &  $.1798\pm.0004$\\
        $\bar{k}/n$ & $.0014\pm.0000$ & $.0030\pm.0000$ & $.0006\pm.0000$ & $.0012\pm.0000$\\
        \bottomrule
    \end{tabularx}
    \label{tab:clustering}
\end{table}

\section{$\sigma$SPMI and Complex Networks}\label{sec:bpmi_cn}

$\sigma$SPMI matrix has the following property: its entries $\sigma\text{SPMI}_{ij}\in[0, 1)$ can be treated as connection probabilities for generating a random graph. As usually, by a graph $\mathcal{G}$ we mean a set of vertices $\mathcal{V}$ and a set of edges $\mathcal{E}\subset\mathcal{V}\times\mathcal{V}$. It is convenient to represent graph edges by its adjacency matrix $(e_{ij})$, in which $e_{ij}=1$ for $(i,j)\in\mathcal{E}$, and $e_{ij}=0$ otherwise. The graph with $\mathcal{V}:=\mathcal{W}$ and $e_{ij}\sim\text{Bernoulli}(\sigma\mathrm{SPMI}_{ij})$ will be referred to as \textit{$\sigma$SPMI-induced Graph}. 

\subsection{Spectrum of the $\sigma$SPMI-induced Graph}
\label{sec:spectrum}
First of all, we look at the spectral properties of the $\sigma$SPMI-induced Graphs.\footnote{We define the graph spectrum as the set of eigenvalues of its adjacency matrix.} For this, we extract SPMI matrices from the {\tt text8} and {\tt enwik9} datasets using the {\sc hyperwords} tool of \citet{levy2015improving}. We use the default settings for all hyperparameters, except the word frequency threshold and context window size. 
We ignored words that appeared less than 100 times and 250 times in \texttt{text8} and \texttt{enwik9} correspondingly, resulting in vocabularies of 11,815 and 21,104 correspondingly. We additionally experiment with the context window size 5, which by default is set to 2. We generate random graphs from the $\sigma$SPMI matrices and compute their eigenvalues using the {\sc TensorFlow} library \citep{abadi2016tensorflow}, and the above-mentioned threshold of 250 for {\tt enwik9} was chosen to fit the GPU memory (11GB, RTX 2080 Ti). The eigenvalue distributions are provided in Figure~\ref{fig:eigvals}. 

The distributions seem to be symmetric, however, the shapes of distributions are far from resembling the Wigner semicircle law $x\mapsto\frac{1}{2\pi}\sqrt{4-x^2}$, which is the limiting distribution for the eigenvalues of many random symmetric matrices with i.i.d. entries \citep{wigner1955characteristic,wigner1958distribution}. This means that the entries of the $\sigma$SPMI-induced graph's adjacency matrix \textit{are} dependent, otherwise we would observe approximately semicircle distributions for its eigenvalues. We observe some similarity between the spectral distributions of the $\sigma$SPMI-induced graphs and of the so-called \textit{complex networks} which arise in physics and network science (Figure~\ref{fig:eigvals}). 

Notice that the connection between human language structure and complex networks was observed previously by \citet{cancho2001small}. A thorough review on approaching human language with complex networks was given by \citet{cong2014approaching}. In the following subsection we will specify precisely what we mean by a complex network.

\subsection{Clustering and Degree Distribution of the $\sigma$SPMI-induced Graph}\label{sec:clust_deg}
We will use two statistical properties of a graph -- degree distribution and clustering coefficient. The \textit{degree} of a given vertex $i$ is the number of edges that connects it with other vertices, i.e. $\deg(i)=\sum_{j\in\mathcal{V}}e_{ij}$. The clustering coefficient measures the average fraction of pairs of neighbors of a vertex that are also neighbors of each other. The precise definition is as follows.

Let us indicate by $\mathcal{G}_i=\{j\in\mathcal{V}\mid e_{ij}=1\}$ the set of nearest neighbors of a vertex $i$. By setting $l_i=\sum_{j\in\mathcal{V}}e_{ij}\left[\sum_{k\in\mathcal{G}_i;\ j<k}e_{jk}\right],
$ we define the local clustering coefficient as $C(i)=\frac{l_i}{{|\mathcal{G}_i|\choose2}}$, and the \textit{clustering coefficient} as the average over $\mathcal{V}$: $C=\frac{1}{n}\sum_{i\in\mathcal{V}}C(i)$.

Let $\bar{k}$ be the average degree per vertex, i.e. $\bar{k}=\frac1n\sum_{j\in\mathcal{V}}e_{ij}$. For random binomial graphs, i.e. graphs with edges $e_{ij}\,\,{\stackrel{\text{iid}}{\sim}}\,\,\mathrm{Bernoulli}(p)$, it is well known \citep{erdHos1960evolution} that 
$C\approx\frac{\bar{k}}{n}$  and $\deg(i)\,\,\sim\,\,\mathrm{Binomial}(n-1,p)$. A \textit{complex network} is a graph, for which $C\gg \frac{\bar{k}}{n}$ and $p(\deg(i)=k)\propto \frac{1}{k^\gamma}$,
where $\gamma$ is some constant \citep{10.5555/1855238}. The latter property is referred to as \textit{scale-free} (or \textit{power-law}) degree distribution.

We constructed $\sigma$SPMI-induced Graphs from the \texttt{text8} and \texttt{enwik9} datasets using context windows of sizes 2 and 5 and ignoring words that appeared less than 5 times,
and computed their clustering coefficients (Table~\ref{tab:clustering}) as well as degree distributions (Figure~\ref{fig:degree_dist}) using the {\sc NetworKit} tool \citep{staudt2016networkit}. {\sc NetworKit} uses the algorithm of \citet{schank2005approximating} to compute the clustering coefficient. 
As we see, the $\sigma$SPMI-induced graphs \textit{are} complex networks, and this brings us to the hyperbolic spaces.

\section{Complex Networks and Hyperbolic Geometry}\label{sec:cn_hg}
Complex networks are ``dual'' to hyperbolic spaces as was shown by \citet{krioukov2010hyperbolic}. They showed that any complex network, as defined in Section~\ref{sec:bpmi_cn}, has an effective hyperbolic geometry underneath. Apart from this, they also showed that any hyperbolic geometry implies a complex network: they placed randomly $n$ points (nodes) into a hyperbolic disk of radius $R$, and used $p_{ij}:=\sigma\left(c[R-x_{ij}]\right)$ as connection probability for connecting nodes $i$ and $j$, where $x_{ij}$ is the hyperbolic distance between $i$ and $j$, and $c$ is a constant. An example of such random graph is shown in Figure~\ref{fig:rhg}.
\begin{figure*}
    \begin{minipage}[t]{.39\textwidth}
    \includegraphics[width=\textwidth]{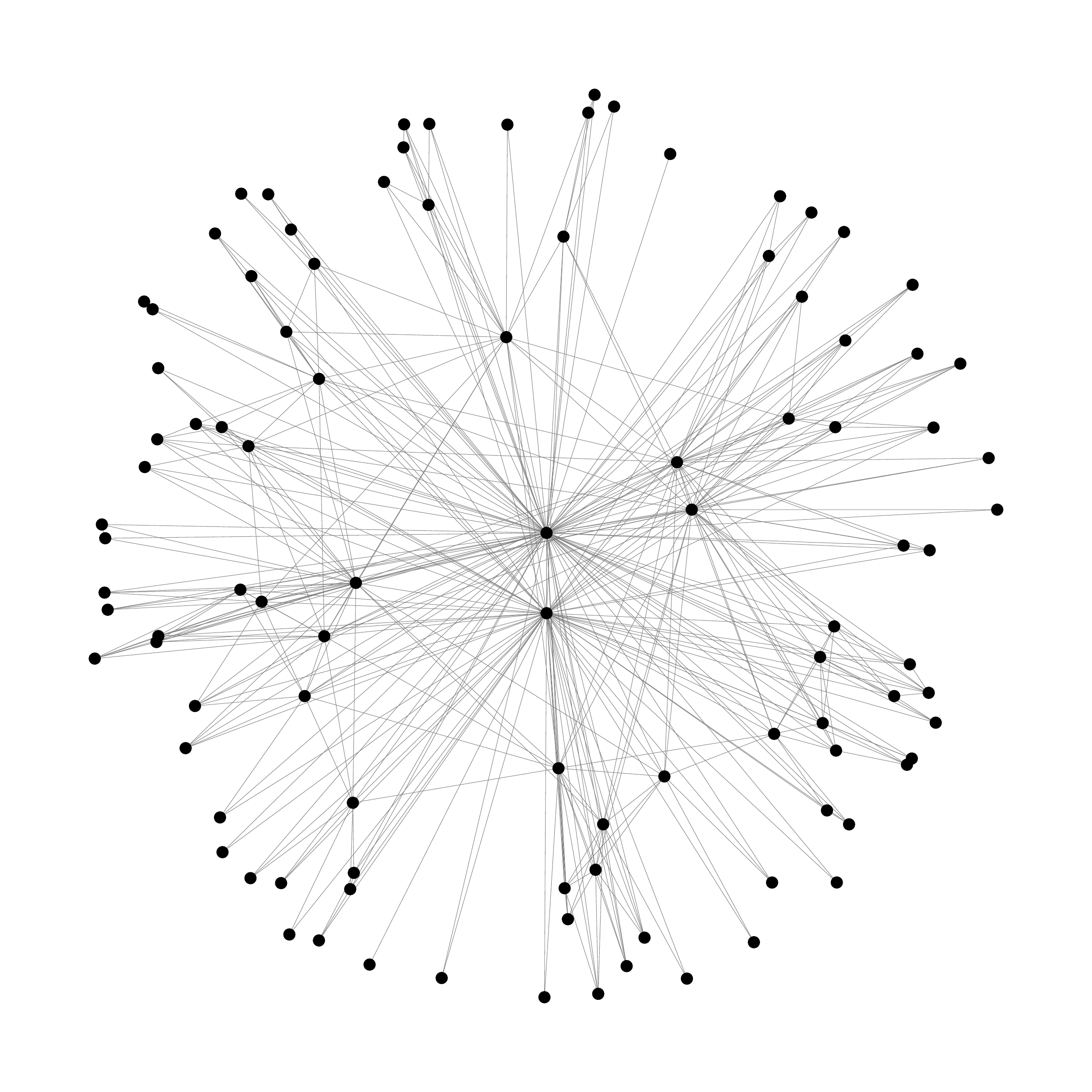}
    \caption{Rand. hyperbolic graph.}
    \label{fig:rhg}
    \end{minipage}\hfill\begin{minipage}[t]{.59\textwidth}\includegraphics[width=\textwidth]{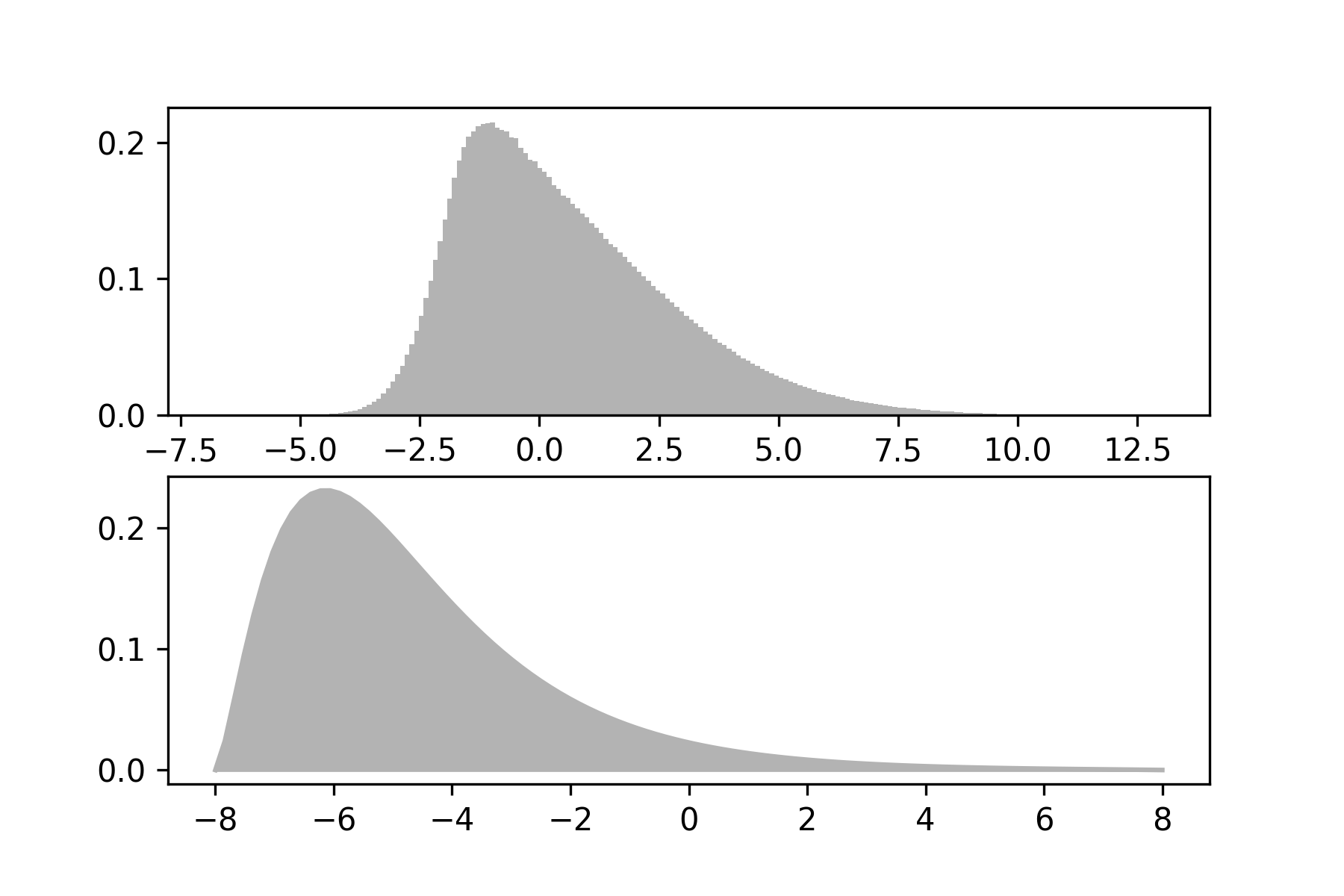}
    \caption{SPMI values distr'n (top) vs $R-X$.}
    \label{fig:spmi_hyperbolic}
    \end{minipage}
\end{figure*}
\citet{krioukov2010hyperbolic} showed that the resulting graph is a complex network. They establish connections between the clustering coefficient $C$ and the power-law exponent $\gamma$ of a complex network and the curvature of a hyperbolic space.

Comparing the construction of \citet{krioukov2010hyperbolic} to the way we generate a random graph from the $\sigma$SPMI matrix, and taking into account that  both methods produce similar structures (complex networks), we conclude that the distribution of the SPMI values should be similar to the distribution of $R-x_{ij}$, i.e. $\text{PMI}_{ij}-\log k\sim R-x_{ij}$. To verify this claim we compare the distribution of SPMI values with the p.d.f. of a random variable $R-X$,  where $X$ is a hyperbolic distance between two random points on the hyperbolic disk (the exact form of this p.d.f.\ is given in the Appendix~\ref{app:dist}). $R$ was chosen according to the formula $R=2\ln[8n/(\pi\bar{k})]$ \citep{krioukov2010hyperbolic}, where $\bar{k}$ is the average degree of the $\sigma$SPMI-induced Graph. The results are shown in Figure~\ref{fig:spmi_hyperbolic}. As we can see, the two distributions are indeed similar and the main difference is in the shift---distribution of $R-X$ is shifted to the left compared to the distribution of the SPMI values. This allows us reinterpreting the pointwise mutual information as the negative of hyperbolic distance (up to scaling and shifting).


\section{Conclusion}
It is noteworthy that the seemingly fragmented sections of scientific knowledge can be closely interconnected. In this paper, we have established a chain of connections between word embeddings and hyperbolic geometry, and the key link in this chain is the Squashed Shifted PMI matrix. Claiming that hyperbolicity underlies word vectors is not novel \citep{nickel2017poincare,tifrea2018poincar}. However, this work is the first attempt to \textit{justify} the connection between hyperbolic geometry and the word embeddings. In the course of our work, we discovered novel objects---Nonsigmoid SGNS and Squashed Shifted PMI matrix---which can be investigated separately in the future.

\section*{Acknowledgements}
This work is supported by the Nazarbayev University faculty-development competitive research grants program, grant number 240919FD3921. The authors would like to thank Zhuldyzzhan Sagimbayev for conducting preliminary experiments for this work, and anonymous reviewers for their feedback.

%
%
%
 \bibliographystyle{splncsnat}
 \bibliography{ref}

\appendix
\section{Auxiliary Results} \label{app:dist}
\begin{proposition} \label{prop:dist_distr} Let $X$ be a distance between two points that were randomly uniformly placed in the hyperbolic disk of radius $R$. The probability distribution function of $X$ is given by
\begin{equation}
    f_X(x)=\int_0^R\int_0^R\frac{\sinh(x)}{\pi\sqrt{1-A(r_1,r_2,x)}\sinh(r_1)\sinh(r_2)}\rho(r_1)\rho(r_2)dr_1dr_2,\label{eq:pdf_x}
\end{equation}
where $A(r_1,r_2,x)=\frac{\cosh(r_1)\cosh(r_2)-\cosh(x)}{\sinh(r_1)\sinh(r_2)}$, and $\rho(r)=\frac{\sinh r}{\cosh R-1}$.
\end{proposition}
The proof is by direct calculation and is omitted due to page limit.

\end{document}